\documentclass[conference]{IEEEtran}
\IEEEoverridecommandlockouts
\usepackage[utf8]{inputenc}
\usepackage[T1]{fontenc}

\usepackage{cite}
\usepackage{amsmath,amssymb,amsfonts, amsthm}
\usepackage{algorithmic}
\usepackage{graphicx}
\usepackage{textcomp}
\usepackage{xcolor}
\usepackage{hyperref}
\usepackage{multirow}

\newcommand{\R}{\mathbb{R}} %% reals

\usepackage{bm} % lettre grecque en gras

\newcommand{\bx}{\mathbf{x}}
\newcommand{\bp}{\mathbf{p}}
\newcommand{\bz}{\mathbf{z}}

\newcommand{\calD}{\mathcal{D}}

\usepackage{graphicx}

\newtheorem{obs}{Observation}

\newtheorem{pro}{Proposition}

\def\BibTeX{{\rm B\kern-.05em{\sc i\kern-.025em b}\kern-.08em
		T\kern-.1667em\lower.7ex\hbox{E}\kern-.125emX}}

\usepackage{quantikz}
\usepackage{comment}
\usepackage{orcidlink}
\begin{document}
	
	\title{Hybrid Variational Quantum Circuits for Multivariate Regression and High-Dimensional Data Reconstruction\\
	
		\thanks{This work was supported by the Agence Française de Développement (AFD) through the Agence Nationale de la Recherche (ANR-21-PEA2-0007), under the Partenariats with "l’Enseignement supérieur
Africain" (PEA) program.}
	}

\author{
\IEEEauthorblockN{
 Koffi O. AYENA
}
\IEEEauthorblockA{
ICB, UTBM \\
%Laboratoire Interdisciplinaire Carnot de Bourgogne \\
%Université de Belfort de Montbéliard\\
F-90000 Belfort, France \\
LAMMA, Universite de Lomé\\
01 BP 1515 Lomé-Togo\\
 ORCID: \href{https://orcid.org/0009-0009-0858-258X}{0009-0009-0858-258X}\\
}
\and

\IEEEauthorblockN{
 Frédéric HOLWECK 
}
\IEEEauthorblockA{
ICB, UTBM \\
F-90000 Belfort, France \\
frederic.holweck@utbm.fr
}
\and

\IEEEauthorblockN{
Serge IOVLEFF 
}
\IEEEauthorblockA{
SINERGIES (UR 4662), UMLP\\
F-90000 Belfort, France \\
serge.iovleff@utbm.fr
}
\and
\IEEEauthorblockN{
Amah S. D'ALMEIDA 
}
\IEEEauthorblockA{
LAMMA, Universite de Lomé\\
01 BP 1515 Lomé-Togo \\
dal\_me@yahoo.fr
}
}

	\maketitle
	
	\begin{abstract}

Variational quantum circuits (VQCs) are parameterized quantum circuits 
optimized classically. We propose a hybrid variational quantum circuit 
(HVQC) extending VQCs with a classical affine post-measurement layer, 
enabling vector-valued regression without the linear overhead of 
independent scalar circuits. Theoretically, we show that elementary 
one- and two-qubit circuits can approximate quadratic functions and 
products via data re-uploading and entanglement, providing the 
foundations of the full architecture. Experimentally, on two synthetic 
image reconstruction datasets and the Friedman1 benchmark (40,568 test 
samples), our HVQC matches Gaussian Process Regression and outperforms 
XGBoost and Random Forest. An ablation study confirms that both quantum 
and classical components are essential, and results highlight the 
central role of the feature map in hybrid quantum-classical models.

	\end{abstract}
	
	
\section{Introduction}
Quantum computing has experienced significant advances in the era of noisy intermediate-scale quantum (NISQ) processors \cite{preskill2018quantum}. These processors, although limited in the number of qubits and subject to noise, have paved the way for new computing paradigms, including variational quantum algorithms (VQAs). Popularized by the foundational work of Peruzzo \cite{peruzzo2014variational} on the variational solver for quantum chemistry, variational quantum circuits (VQCs) have become central models for exploring the potential advantages of quantum computing in various fields, notably quantum chemistry and, more recently, quantum machine learning (QML)~\cite{biamonte2017quantum, cerezo2021variational}.

The enthusiasm for variational quantum circuit  (HVQC) stems from their hybrid architecture: a parameterized quantum circuit, whose parameters are optimized by a classical computer. This approach partially circumvents the limitations of current quantum hardware. Formally, a VQC can be described as a parameterized unitary $U(\bz; \bm{\theta})$ applied to an initial state $\ket{\mathbf{0}}=\lvert 0 \rangle^{\otimes^{m_0}}$, producing a final state whose expectation value of an observable $\hat{O}$ provides the model's output \cite{benedetti2019parameterized}:
\begin{equation}
	f(\bz; \bm{\theta}) = \bra{\mathbf{0}} U^{\dagger}(\bz; \bm{\theta}) \hat{O} U(\bz; \bm{\theta}) \ket{\mathbf{0}}
	\label{eq:vqc}
\end{equation}
where $\mathbf{z} \in \mathbb{R}^d$ represents the input data, typically encoded via parameterized rotation gates \cite{schuld2021effect}.

In the QML landscape, while classification tasks \cite{schuld2019quantum, perez2020data, schuld2020circuit}, and "unpublished" \cite{lloyd2020quantum} have been extensively studied, \textit{quantum regression} has remained a theoretically less explored topic until recently, as highlighted by authors in \cite{garate2024variational} on applying QML to practical regression on NISQ hardware. Yet, the ability to predict continuous values is crucial for numerous scientific and industrial applications, including financial time series forecasting \cite{orus2019quantum}, physical data modeling \cite{ciliberto2018quantum}, or climate prediction \cite{da2026exploring}.

Research on quantum regression circuits has experienced a remarkable acceleration since 2024, marked by several important milestones which we organize chronologically.

The year 2024 constitutes a turning point with the first significant experimental demonstrations. A pioneering study \cite{garate2024variational} applied variational quantum regression to the Auto-MPG dataset on NISQ hardware with error mitigation. Their results demonstrate that VQAs can outperform classical models like XGBoost, and that error mitigation techniques are effective in bringing the performance of noisy simulators closer to that of ideal simulators.

Concurrently, the PennyLane platform published a tutorial demonstrating multidimensional regression with a two-qubit variational circuit to approximate the function $f(x_1,x_2)=\frac{1}{2}(x_1^2+x_2^2)$, achieving an $R^2$ score of 0.983 in "unpublished" \cite{JorgeGorkaMartinezdeLejarza2024}. This work illustrates the ability of VQCs to build partial Fourier series for function approximation, consistent with the VQC expressivity theory established in \cite{schuld2021effect}.

Reference \cite{senapati2024pqml} introduced the PQML (Predictive Quantum Machine Learning) tool to predict the reproducibility of results across different quantum machines, a crucial advance for the reliability of quantum regression applications in a heterogeneous NISQ context.

The year 2025 saw the emergence of sophisticated optimization techniques and applications to complex problems. In \cite{perkkola2025optimizing}, the authors proposed a novel state preparation method for variational quantum regression, using optimization techniques based on the ZX-calculus (Pauli pushing, phase folding, Hadamard pushing). Their results demonstrate that these optimizations enable the successful execution of quantum regression algorithms on current hardware, significantly reducing circuit depth.

This approach was extended to multivariate time series in "unpublished"\cite{hamhoum2025multivariate} through the MTS-QRC (Multivariate Time Series Quantum Reservoir Computing) framework.

Applied to Lorenz and ENSO (El Niño-Southern Oscillation) systems, this method achieved a mean squared error (MSE) of $0.0087$ and $0.0036$, respectively. Interestingly, their work revealed that hardware noise can sometimes act as an implicit regularizer, improving performance compared to ideal simulators a counter-intuitive yet promising phenomenon for NISQ applications.

Parallel advances in error mitigation directly benefited regression applications. The authors of \cite{czarnik2021error} proposed significant improvements to Clifford Data Regression (CDR) with Energy Sampling (ES) and Non-Clifford Extrapolation (NCE), enhancing the fidelity of computations on noisy hardware without additional quantum overhead.

Automatic design of quantum architectures for regression using genetic algorithms was explored in "unpublished" \cite{neto2025regression}. Their \textit{Reduced Regressor QNN} framework explores circuit depth, configuration of parameterized gates, and data re-uploading patterns, demonstrating that these evolved circuits, although compact, can achieve competitive performance against 17 classical regression models on 22 non-linear benchmark functions.

The year 2026 marks a consolidation of the field. Joo's editorial \cite{joo2026advancing} in \textit{Frontiers in Physics} reviews the progress in algorithm optimization and error mitigation, confirming that these areas are now mature enough to support practical applications such as quantum regression. The challenges are progressively shifting from fundamental feasibility towards comparative efficiency and demonstrable quantum advantage.

Although variational quantum circuits (VQCs) have shown promise for scalar regression, they suffer from the barren plateau phenomenon \cite{mcclean2018barren, cunningham2025investigating} that makes training difficult, and their naive extension to vector-valued regression incurs a prohibitive linear complexity in the output dimension. To address these limitations, we introduce a hybrid variational quantum circuit (HVQC) that handles vector-valued regression in a single unified circuit, thereby avoiding the linear overhead of independent approaches \cite{garate2024variational, JorgeGorkaMartinezdeLejarza2024}. Unlike quantum kernel methods \cite{havlivcek2019supervised, schuld2021supervised} where the circuit is fixed, our architecture jointly optimizes all quantum and classical parameters end-to-end, with the final affine layer enabling the model to reach any target in the output vector space.
%\subsection{Scientific Challenges and Current Limitations}
%While VQCs have shown promise for scalar regression ($\mathbb{R}^d \rightarrow \mathbb{R}$), several fundamental challenges persist. The barren plateau phenomenon, studied in \cite{mcclean2018barren, cunningham2025investigating}, stands as the primary obstacle to training deep VQCs: under conditions of high expressivity or excessive depth, the gradient variance decays exponentially with the number of qubits, effectively preventing gradient-based learning. Understanding and mitigating this phenomenon in regression tasks remains an open problem. This challenge becomes especially critical for vector-valued regression ($\mathbb{R}^d \rightarrow \mathbb{R}^m$), where naive strategies based on assembling independent scalar circuits yield linear complexity in $m$, resulting in a prohibitive cost under limited quantum resources.

%To address these limitations, we introduce a hybrid variational quantum circuit for vector-valued regression. The proposed architecture leverages fundamental building blocks, namely $R_Y$ rotations and entanglement patterns to harness the expressivity of shallow circuits.

This paper is organized as follows. Section~\ref{sec2} presents the mathematical foundations of HVQCs. In Section~\ref{sec3}, we demonstrate the approximation capabilities of elementary one- and two-qubit circuits, thereby providing the theoretical foundations for our architecture.
Complete HVQC architecture is established in Section~\ref{sec4}, where we define the key components: data encoding, and the variational parameterization that enables learning. 
Section~\ref{sec5} presents experimental results comparing the performance of our HVQC, evaluated using the $R^2$ score and mean squared error (MSE), on two simulated datasets with different feature maps, against several classical regression models, including Gaussian process regression (GPR), random forest regression (RFR),  multi-output XGBoost Regression (XGB), and two classical neural networks. The section concludes with a global analysis of the quantum states after measurement, revealing an implicit clustering behavior.
\section{Theoretical framework of a HVQC}
\label{sec2}
We consider the following supervised learning problem: given a training set $\calD = \{(\bz^{(i)}, \bx^{(i)})\}_{i=1}^N$, where $\bz^{(i)} \in \R^d$ and $\bx^{(i)} \in \R^m$, the objective is to learn a function $f_{\bm{\Theta}}: \R^d \rightarrow \R^m$, parameterized by $\bm{\Theta}$, minimizing the mean squared error:
\begin{equation}
\label{eq2}
    \mathcal{L}(\bm{\Theta}) = \frac{1}{N} \sum_{i=1}^N \| f_{\bm{\Theta}}(\bz^{(i)}) - \bx^{(i)} \|_2^2. 
\end{equation}

\paragraph{Baseline architecture of the quantum regressor}
The conventional architecture of the variational quantum regressor  proceeds in three steps. First, a quantum feature map embeds the input vector $\bz$ into the Hilbert space of a system of $m_0 \geq \lceil \log_2 m\rceil$ qubits, producing the state $\ket{\phi(\bz)} \in \mathcal{H} \cong (\mathbb{C}^2)^{\otimes m_0}$. In second step, this encoding, whose dimensionality is chosen to match the target dimension $m$, is then evolved by a parameterized quantum circuit $U(\bm{\theta})$, generating the variational state:
\begin{equation}
    \ket{\psi(\bz;\bm{\theta})} = \phi(\bz)U(\bm{\theta}^{L})\cdots \phi(\bz)U(\bm{\theta}^{1}) \phi(\bz)\ket{0}^{\otimes m_0}. \label{eq:variational_state}
\end{equation}

The model's prediction emerges from measuring the $m_0$ qubits in the computational basis $\{|k\rangle\}_{k=0}^{2^{m_0}-1}$. The theoretical probability distribution of outcomes is given by the projectors $M_k = |k\rangle\langle k|$:
\begin{equation}
    p_k(\bz; \bm{\theta}) = \langle \psi(\bz; \bm{\theta}) | M_k | \psi(\bz; \bm{\theta}) \rangle. \label{eq:quantum_probs}
\end{equation}
This distribution $\bp(\bz; \bm{\theta}) = (p_0, \ldots, p_{2^{m_0}-1})$ resides in the probability simplex $\Delta^{2^{m_0}-1}$, defined by:
\begin{align*}
    &\Delta^{2^{m_0}-1} =\\ 
    &\left\{ (q_0,q_1,\cdots,q_{2^{m_0}-1} )\in \mathbb{R}^{2^{m_0}} \mid q_k \geq 0, \sum_{k=0}^{2^{m_0}-1} q_k = 1 \right\}. \label{eq:simplex}
\end{align*}
In practice, access to this distribution is obtained through sampling (shots). For $S$ measurement repetitions, one obtains a frequentist estimate $\hat{p}_k(\bm{\theta}, \bz) = c_k/S$, where $c_k$ is the count of outcome $k$.

Each projector $M_k = |k\rangle\langle k|$ defines an observable whose measurement yields the probability of observing the computational basis state $|k\rangle$. It can be expressed as a tensor product of single-qubit observables

\[
M_k = \bigotimes_{i=1}^{m_0} \frac{1}{2}\Big(I + (-1)^{k_i} \sigma_z\Big), \qquad \text{with} \quad \sigma_z= \begin{pmatrix} 1 & 0 \\ 0 & -1 \end{pmatrix}
\]
where $k_i \in \{0,1\}$ denotes the $i$-th bit of the integer $k$ on $m_0$ bits, $\sigma_z$ is the Pauli $Z$  matrix and $I$ denotes the  identity matrix.

In the third step, we perform post-processing to handle model limitations and ensure that the output dimensions are consistent. This final step is discussed in the following two paragraphs.

\paragraph{Geometric limitation of the baseline model}
A fundamental limitation of the architecture defined by equations \eqref{eq:variational_state} and \eqref{eq:quantum_probs} lies in the confinement of its output to the probabilistic simplex $\Delta^{2^{m_0}-1}$. Although a trivial linear post-processing could theoretically project this output onto $\mathbb{R}^m$, the \textit{internal geometry} of the learned representations remains that of a convex polytope with extremal properties — its points being convex combinations of computational basis states. This geometric constraint intrinsically limits the model's capacity to capture data structures exhibiting a different underlying geometry (unbounded, non-trivial topology), necessitating an appropriate preprocessing of input data.

\paragraph{Extension via post-variational affine transformation}
To overcome this limitation while preserving the model's differentiability, we propose incorporating a learnable affine transformation of the probability distribution $\bp(\bm{\theta}, \bz)$, like "unpublished" \cite{wilson2018quantum}. The complete regression function is then written as:
\begin{equation}
    f_{\bm{\Theta}}(\bz) = \mathbf{W} \, \bp(\bm{\theta}, \bz) + \bm{b}, \label{eq:affine_vqr}
\end{equation}
where $\bm{\Theta} = \{\bm{\theta}, \mathbf{W}, \bm{b}\}$ encompasses all parameters. The matrix $\mathbf{W} \in \mathbb{R}^{m \times 2^{m_0}}$ and the bias vector $\bm{b} \in \mathbb{R}^m$ perform a fundamental geometric transformation: they allow the model to reach any point in $\mathbb{R}^m$ by learning a linear map from the simplex to the target space. This formulation breaks the simplexial geometry of internal representations and offers a clear geometric interpretation, conferring upon the model increased flexibility to adapt to complex data patterns.

To better understand the capabilities of our model, we now turn to the expressivity of HVQCs, focusing on their ability to approximate non-linear functions.

\section{Expressivity of VQCs}
\label{sec3}
%In this section, we demonstrate the capacity of small one- or two-qubit quantum circuits to approximate non-linearity.

We demonstrate here the capacity of elementary one- and two-qubit quantum circuits to approximate fundamental nonlinearities, thereby establishing the theoretical basis for our architecture.
We validate Propositions \ref{pro1} and \ref{pro2} via statevector simulation using PennyLane, neglecting shot noise in the probability estimation. 
\begin{pro}[Approximation of the square function]
	\label{pro1}
	Let $K =[-2, 2]$. For any $\varepsilon \in (0,1)$, there exist $\beta < 1$ and $k\geq 1$, such that the measurement of the quantum state $\ket{\psi(z)}=R_Y(z)\ket{0}$ in the computational basis satisfies
	\begin{equation}
    \label{eq:6}
		\sup_{z\in K} \left| z^2 - 4\beta^{-2k}P\left(\ket{\psi(\beta^k z)}=\ket{1}\right) \right| \leq \varepsilon
	\end{equation}
\end{pro}

\begin{proof}
	We have the quantum state $\ket{\psi(z)} = R_y(z)\ket{0}$. Computing explicitly, we obtain:
	\[
	\ket{\psi(z)} = \begin{pmatrix} 
		\cos\frac{z}{2}  \\
		\sin\frac{z}{2}
	\end{pmatrix}
	\]
	The probability of measuring $\ket{1}$ is:
	\begin{align}
		P\left(\ket{\psi(z)}=\ket{1}\right) &= \left|\sin\frac{z}{2}\right|^2
	\end{align}
	Hence,
	\[
	4\beta^{-2k}P\left(\ket{\psi(\beta^k z)}=\ket{1}\right) = 4\beta^{-2k}\sin^2\frac{\beta^k z}{2}
	\]
	
	For small $u$, we have the Taylor expansion $\sin^2\frac{u}{2} = \frac{u^2}{4} - \frac{u^4}{48} + O(u^6)$. Here $u = \beta^k z$ is small for large $k$. Thus:
	\[
	4\beta^{-2k}\sin^2\frac{\beta^k z}{2} = 4\beta^{-2k}\left(\frac{\beta^{2k}z^2}{4} - \frac{\beta^{4k}z^4}{48} + O(\beta^{6k}z^6)\right)
	\]
	\[
	= z^2 - \frac{\beta^{2k}z^4}{12} + O(\beta^{4k}z^6)
	\]
	
	By Taylor's theorem with the Lagrange remainder, for any $u \in \mathbb{R}$, we have
	\[
	\left|\sin^2\frac{u}{2} - \frac{u^2}{4}\right| \leq \frac{|u|^4}{48}
	\]
	
	Applied to $u = \beta^k z$:
	\[
	\left|4\beta^{-2k}\sin^2\frac{\beta^k z}{2} - z^2\right| \leq 4\beta^{-2k}\frac{|\beta^k z|^4}{48} = \frac{\beta^{2k}|z|^4}{12}
	\]
	
	Then for all $z \in K$:
	\[
	\left|4\beta^{-2k}P\left(\ket{\psi(\beta^k z)}=\ket{1}\right) - z^2\right| \leq \frac{4\beta^{2k}}{3}
	\]
	
	For sufficiently large $k$, $\beta^{2k}$ becomes arbitrarily small (since $\beta < 1$), so this difference tends to 0 uniformly on $K$.
	
	Thus, for any $\varepsilon \in (0,1)$, it suffices to choose $\beta \leq 1$ and $k$ large enough such that $\frac{4\beta^{2k}}{3} \leq \varepsilon$, which yields:
	\[
	\sup_{z\in K} \left| z^2 - 4\beta^{-2k}P\left(\ket{\psi(\beta^k z)}=\ket{1}\right) \right| \leq \varepsilon
	\]
\end{proof}
\begin{figure}[htbp]
	\centerline{\includegraphics[scale=0.36]{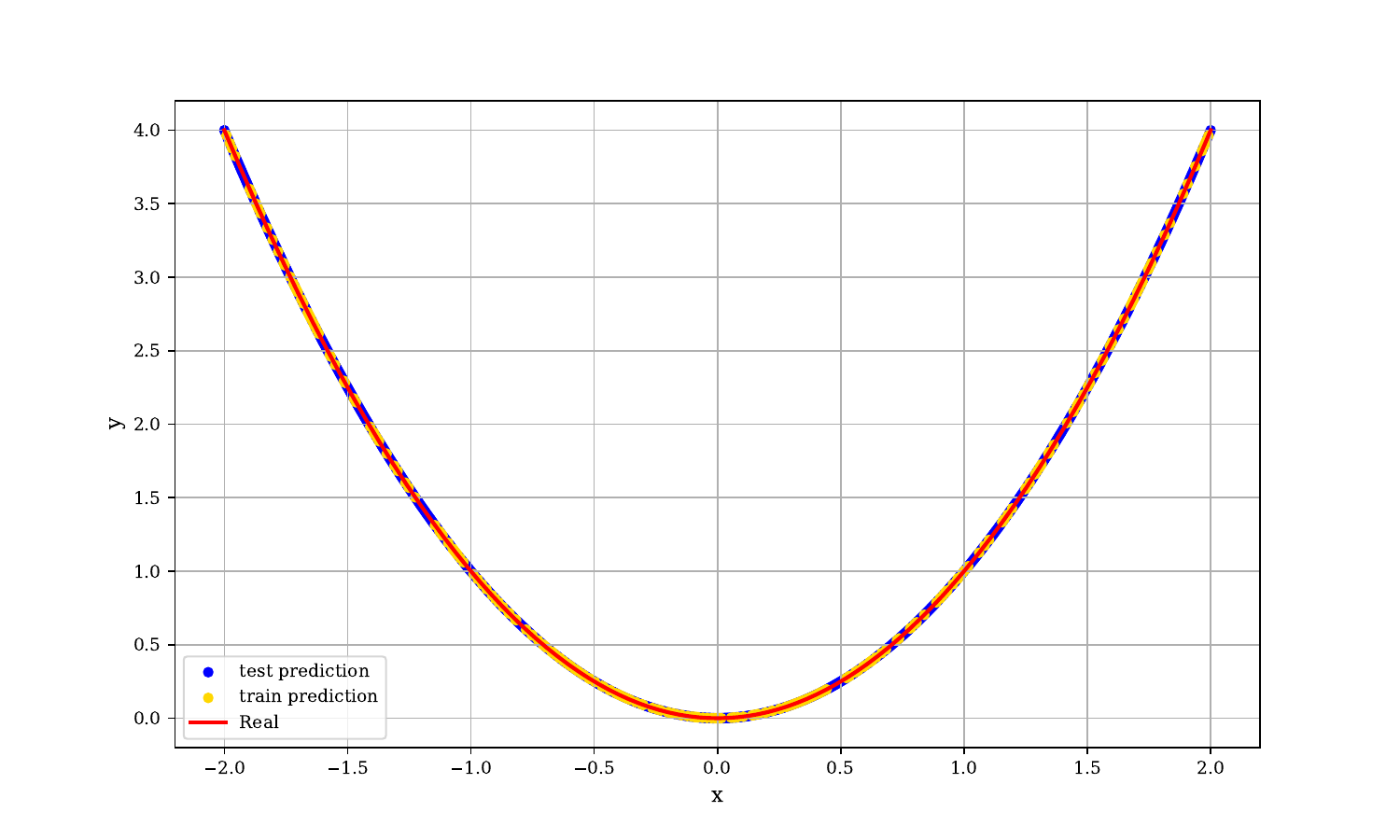}}
	\caption{Square function approximation.}
	\label{fig1}
\end{figure}
The approximation of the function $x^2$ by the HVQC defined in Proposition~\ref{pro1} achieves an R$^2$ score of 1 on the training set and 1 on the test set, using 900 points split in a 70\%–30\% proportion, as illustrated in Fig.~\ref{fig1}.
Using the identity
$$x\times y=\frac{1}{4} \left[(x+y)^2-(x-y)^2\right]$$
combined with the square function approximation obtained in Proposition~\ref{pro1}, we can then approximate the product $x \times y$ using two $R_Y$ rotation gates applied to two qubits.

\begin{pro}[Quantum approximation of the product]
	\label{pro2}
	For any $\varepsilon \in (0,1)$, there exist parameters $\beta < 1$, $k \geq 1$, and a quantum circuit acting on a 2-qubit system such that measuring a quantum state yields an approximation $\widetilde{xy}$ satisfying:
	
	\begin{equation}
		\sup_{(x,y)\in [-1, 1]^2} \left| x \times y - \widetilde{xy} \right| \leq \varepsilon
	\end{equation}
	
	More precisely, the approximation is given by:
	
	\begin{align}
    \nonumber
		&\widetilde{xy} =
		 \beta^{-2k}\times\\ 
		 &\left[ P\left(\ket{\psi(\beta^k(x+y))}=\ket{1}\right) - P\left(\ket{\psi(\beta^k(x-y))}=\ket{1}\right) \right]
		\label{eq:product_approx}
	\end{align}
	
	where $\ket{\psi(z)} = R_Y(z)\ket{0}$ is the quantum state defined in Proposition~\ref{pro1}.
\end{pro}

\begin{proof}
	From Proposition~\ref{pro1}, for any $z \in K$, we have:
	
	\begin{equation}
		\left| z^2 - 4\beta^{-2k}P\left(\ket{\psi(\beta^k z)}=\ket{1}\right) \right| \leq \frac{4\beta^{2k}}{3}
	\end{equation}
	
	Since $x+y, x-y \in K$, and using the remarkable identity:
	
	\begin{equation}
		x \times y = \frac{1}{4}\left[(x+y)^2 - (x-y)^2\right]
	\end{equation}
	The approximation error follows from the triangle inequality:
	
	\begin{align}
    \nonumber
		&\left| x \times y - \widetilde{xy} \right|\\
		\nonumber
        &\leq \frac{1}{4}\Big[ \left| (x+y)^2 - 4\beta^{-2k}P\left(\ket{\psi(\beta^k(x+y))}=\ket{1}\right) \right| \\
		\nonumber
        &\quad + \left| (x-y)^2 - 4\beta^{-2k}P\left(\ket{\psi(\beta^k(x-y))}=\ket{1}\right) \right| \Big] \\
		&\leq \frac{1}{4}\left[ \frac{4\beta^{2k}}{3} + \frac{4\beta^{2k}}{3} \right] = \frac{2\beta^{2k}}{3}
	\end{align}
	
	For any $\varepsilon \in (0,1)$, it therefore suffices to choose $\beta < 1$ and $k$ sufficiently large such that:
	
	\begin{equation}
		\frac{2\beta^{2k}}{3} \leq \varepsilon
	\end{equation}
\end{proof}
The corresponding quantum circuit can be implemented by preparing the states $\ket{\psi(\beta^k(x+y))}$ and $\ket{\psi(\beta^k(x-y))}$ on two distinct qubits, which allows obtaining both probabilities in execution of the circuit.

On four qubits, Proposition~\ref{pro2} with CNOT gates can be rewritten as:
\begin{equation}
\label{eq:14}
    \begin{aligned}
    & \bigl(R_Y(+\pi/2) \otimes I \otimes R_Y(-\pi/2) \otimes I\bigr) \ \circ \bigl(\mathrm{CNOT}_{0, 1}\\
    &\otimes \mathrm{CNOT}_{2, 3}\bigr) 
    \ \circ \bigl(R_Y(\beta^k x)^{\otimes 2} \otimes R_Y(\beta^k y)^{\otimes 2}\bigr) 
    \ |0000\rangle,
    \end{aligned}
\end{equation}
with
    $
        CNOT_{i,j}|q_i q_j\rangle
	 \rightarrow 
	|q_i , (q_j \oplus q_i)\rangle$ and 
    $\oplus$ means addition modulo 2 (classical XOR).
And denoting by $P_{01}$ and $P_{23}$ the probabilities of measuring $|11\rangle$ 
on qubits $(0,1)$ and $(2,3)$ respectively, $\frac{2}{\beta^{2k}}\bigl(P_{01} - P_{23}\bigr)$ approximates $xy$.
    We further observe the following
\begin{obs}[Approximations realized by the two-qubit circuit]
	\label{obs}
	Consider the quantum circuit $\mathcal{C}(\alpha)$ acting on two qubits defined by:
	
	\begin{align}
    \nonumber
		&\mathcal{C}(\alpha) : \ket{00} \mapsto \ket{\psi(x,y,\alpha)} =\\
		& CNOT_{0, 1} \cdot(I \otimes R_Y(\alpha)) \cdot  (R_Y(x) \otimes R_Y(y)) \ket{00}
		\label{eq:two_qubit_circuit}
	\end{align}
    For small values of $x$ and $y$ (in the neighborhood of 0), the measurement probabilities in the computational basis realize the following approximations:
	
	\begin{align}
		P_{00}(x,y,0)= P_{01}(x,y,\pi) &\approx 1-\frac{x^2}{4} -\frac{y^2}{4}+\frac{x^2 y^2}{16}\\
		P_{01}(x,y,0) = P_{00}(x,y,\pi) &\approx \frac{y^2}{4} -\frac{x^2 y^2}{16} \\
		P_{10}(x,y,0) =P_{11}(x,y,\pi) &\approx \frac{x^2 y^2}{16} \\
		P_{11}(x,y,0) =P_{10}(x,y,\pi)&\approx \frac{x^2}{4}-\frac{x^2 y^2}{16}
		\label{eq:prob_approximations}
	\end{align}
	where $P_{ij}(x,y,\alpha)$ is the probability of measuring the state $\ket{ij}$.
\end{obs}
\begin{figure}[htbp]
	\centerline{\includegraphics[scale=0.29]{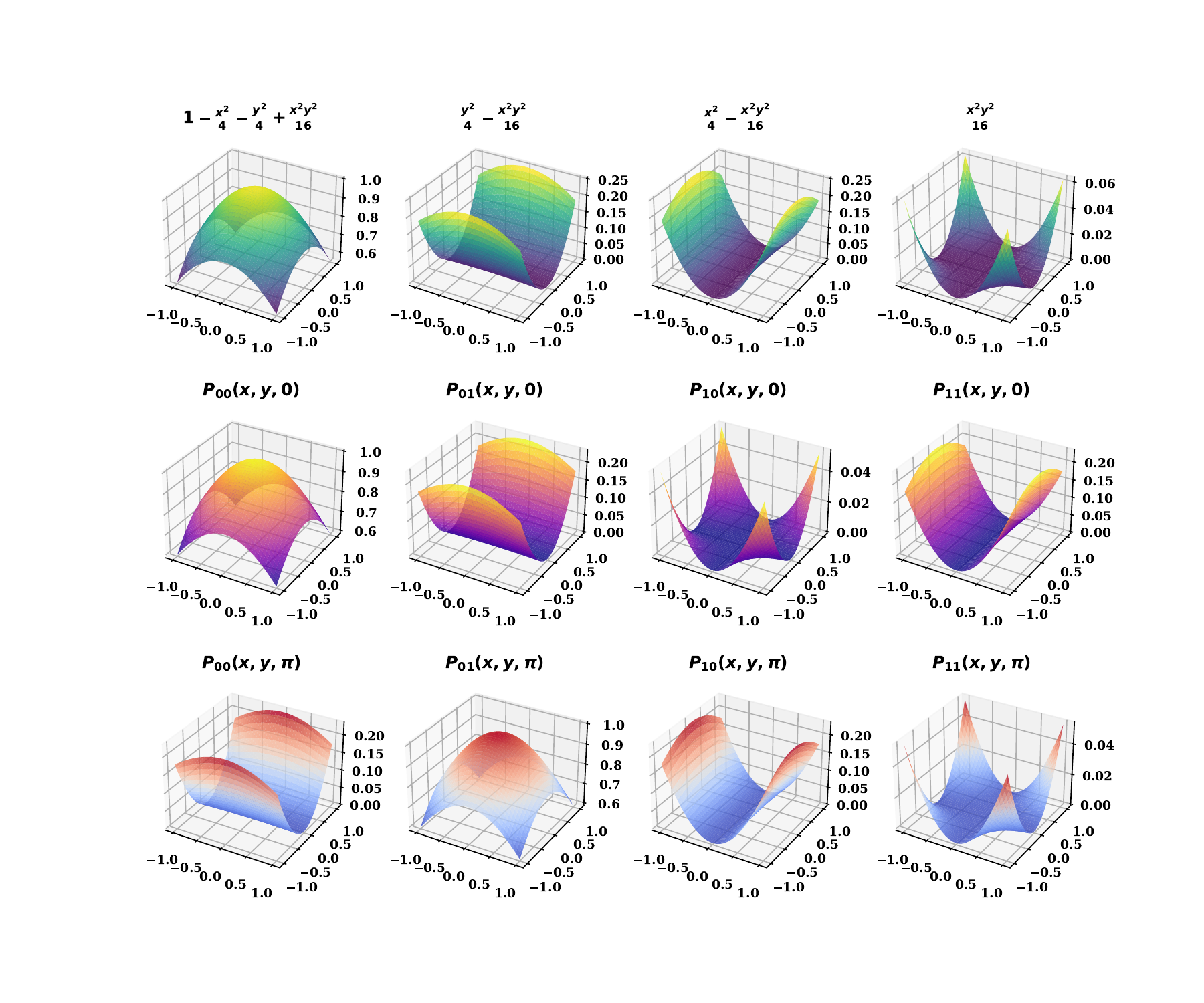}}
	\caption{The measurement probabilities of $\ket{\psi(x,y,\alpha)}$ in the computational basis states $\ket{ij}, i, j \in \{0,1\}$.}
	\label{fig2}
\end{figure}

In Observation~\ref{obs}, entanglement is preserved only when $R_Y(\alpha)$ precedes the CNOT, yielding a non-zero determinant $\frac{1}{2}\sin x \cos(y+\alpha)$. The only exceptions are when $\sin x = 0$ or $\cos(y+\alpha)=0$ (first qubit in a computational basis state).
%A direct calculation, in Observation \ref{obs}, shows that the position of the variational layer relative to the CNOT critically determines whether entanglement is preserved. When the CNOT is placed before the variational rotation $R_Y(\alpha)$, the final state remains separable. Conversely, placing $R_Y(\alpha)$ before the CNOT (as in the original ordering) yields a generically entangled state, as evidenced by a non-zero determinant $\frac{1}{2}\sin x \cos(y+\alpha)$. The only exceptions occur when $\sin x = 0$ or $\cos(y+\alpha)=0$, i.e., when the first qubit lies in a computational basis state, for example. 
%This observation underscores the importance of circuit architecture in exploiting quantum correlations for learning non-separable functions, emphasized by several authors \cite{havlivcek2019supervised, perez2020data} and "unpublished" \cite{ lloyd2020quantum}.

\section{Proposed HVQC Tools and Architecture}
\label{sec4}

Three design principles follow from Propositions~\ref{pro1} and~\ref{pro2}: (i)~\textbf{data 
re-uploading} — re-encoding $\bz'$ at each layer~\eqref{encoding} builds complex dependencies like $\beta^k z$ in \eqref{eq:6}; 
(ii)~\textbf{entanglement} — the cascade CNOT~\eqref{entangl} generalizes~\eqref{eq:14} to all feature dimensions; 
(iii)~\textbf{affine post-processing} —~\eqref{target} generalizes 
the rescaling $4\beta^{-2k}$ in~\eqref{eq:6} to arbitrary linear combinations, lifting the simplex constraint.

This section details the complete architecture of our hybrid variational quantum circuit.
%, specifying each component: data encoding, parameterized circuit, measurement, and affine post-processing.
\subsection{Hilbert space and data encoding}

Let \(\mathcal{H} = (\mathbb{C}^2)^{\otimes m_0}\) be the Hilbert space associated with the $m_0$-qubit system, of dimension \(\dim(\mathcal{H}) = 2^{m_0} \geq m\). A quantum state \(\ket{\psi} \in \mathcal{H}\) can be decomposed in the computational basis \(\{\ket{i}\}_{i=0}^{2^{m_0}-1}\) as:

\begin{equation}
	\ket{\psi} = \sum_{i=0}^{2^{m_0}-1} \alpha_i \ket{i}, \quad \sum_{i=0}^{2^{m_0}-1} |\alpha_i|^2 = 1
\end{equation}

The encoding of classical data is achieved by the mapping:

\begin{equation}
	F : \mathbb{R}^d \rightarrow \mathbb{R}^{m_0} \quad F(\mathbf{z}) = \mathbf{z}'
\end{equation}
which augments the $d$ components of $\bz$ to $m_0$ components of $\bz'$ (padding encoding)
with $$\mathbf{z}'=(z_1, z_2,\cdots, z_d, c_0, \ldots, c_{m_0 -d-1}) \in \mathbb{R}^{m_0}$$
and 
\begin{equation}
\label{encoding}
	\Phi : \mathbb{R}^{m_0} \rightarrow \mathcal{H}, \quad \Phi(\mathbf{z}') = \bigotimes_{k=1}^{m_0} R_Y(\mathbf{z}'_k)\ket{\psi}
\end{equation}

where \(R_Y(\theta) = e^{-i\frac{\theta}{2}Y}\) is the rotation operator about the Y-axis, whose matrix representation in the computational basis is:

\begin{equation}
	R_Y(\theta) = \begin{pmatrix} 
		\cos(\theta/2) & -\sin(\theta/2) \\
		\sin(\theta/2) & \cos(\theta/2)
	\end{pmatrix}
\end{equation}

It can be noted that the coefficients \(c_0, \ldots, c_{m_0 - d - 1}\) may be constants, or alternatively that \(\mathbf{z}' = \mathbf{z}^{\otimes m_0}\) for a suitably chosen integer \(m_0\), as detailed in reference \cite{schuld2020circuit} (tensor product encoding). Furthermore, the components of \(\mathbf{z}'\) may also depend on the \(z_i\) (non-linear feature encoding in Fig.~\ref{fig3}); in this case, it is shown that this can improve learning (Table~\ref{tab1}).

\subsection{Architecture of the parameterized circuit}

The quantum circuit implements a parameterized unitary transformation \(U(\bz'; \bm{\theta}) : \mathcal{H} \rightarrow \mathcal{H}\). Let $L$ be the number of layers. The architecture alternates encoding layers \(\Phi(\mathbf{z}')\) and strongly entangling layers \(V(\bm{\theta}_\ell)\) 
\begin{align}
	\label{architecture}
        \nonumber
	U(\bz'; \bm{\theta}) = V(\bm{\theta}_L) \circ \Phi(\mathbf{z}') \circ V(\bm{\theta}_{L-1}) \circ \Phi(\mathbf{z}') \circ \cdots 
        &\\ \circ V(\bm{\theta}_1) \circ \Phi(\mathbf{z}') \circ V(\bm{\theta}_0)
\end{align}

\subsubsection{Strongly entangling layer}

For the $\ell$-th layer, \(V(\bm{\theta}_\ell)= U_{\text{ent}} \cdot R(\bm{\theta}_\ell)\), with the order suggested by Observation~\ref{obs} decomposes into two parts:

\begin{itemize}
	\item \textbf{Local rotations} on each qubit:
	\begin{equation}
		R(\bm{\theta}_\ell) = \bigotimes_{k=1}^{m_0}
		R_X(\theta_{\ell,k}^{(1)}) R_Y(\theta_{\ell,k}^{(2)}) R_Z(\theta_{\ell,k}^{(3)})
	\end{equation}
	
	\item \textbf{Entangling CNOT gates} in a cascade pattern:
	\begin{equation}
    \label{entangl}
		U_{\text{ent}} = \prod_{i=0}^{m_0-1}	CNOT_{i,(i+r)\bmod m_0}.
	\end{equation}
\end{itemize}
with
    \begin{align}
    \nonumber
        CNOT_{i,j}|q_0 \dots q_i \dots q_j \dots q_{m_0-1}\rangle
	 \rightarrow \\
	|q_0 \dots q_i \dots (q_j \oplus q_i) \dots q_{m_0-1}\rangle
    \end{align}
and $r$ is a hyperparameter called the range which defines the distance between the control qubit $i$ and the target qubit $j$.

\begin{figure}[htbp]
\centering
\includegraphics[width=0.65\linewidth]{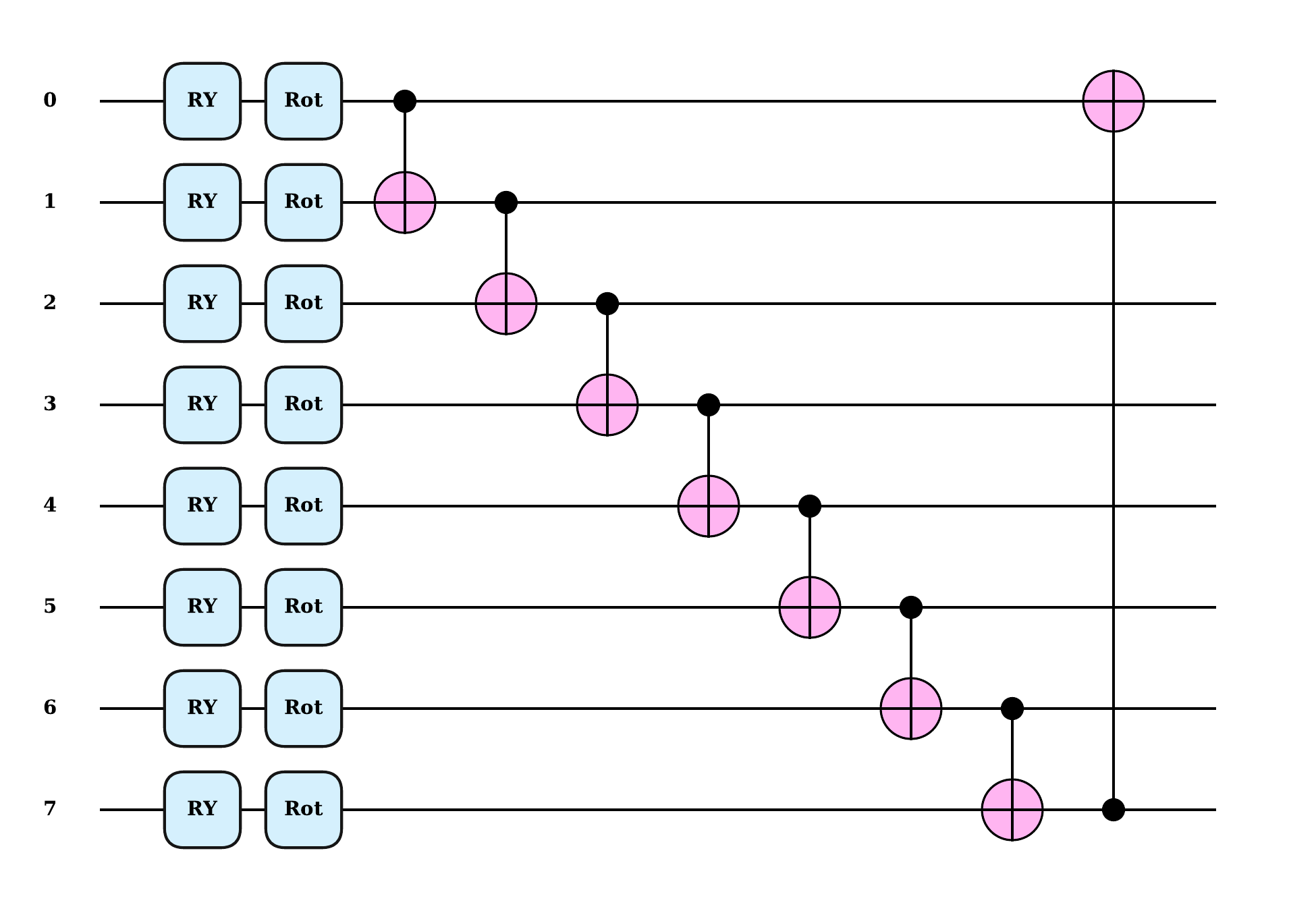}
%\hfill \quad
%\includegraphics[width=0.65\linewidth]{VQC_F3_Data1.pdf}

\caption{Architecture of a single layer of the VQC. $Rot$ denotes $R_X(\theta_{\ell,k}^{(1)}) R_Y(\theta_{\ell,k}^{(2)}) R_Z(\theta_{\ell,k}^{(3)})$.}
\label{fig3}
\end{figure}
\subsection{Measurement and probability distribution}

Measurement in the computational basis yields a probability vector \(\mathbf{p}(\mathbf{z}'; \bm{\theta}) \in \mathbb{R}^{2^{m_0}}\) belonging to the simplex \(\Delta^{2^{m_0}-1}\)
Each component is expressed as:
\begin{equation}
    \label{eq35}
		p_j(\mathbf{z}'; \bm{\theta}) = \text{Tr}\left( \ket{j}\bra{j} \, \rho(\mathbf{z}'; \bm{\theta}) \right)
\end{equation}
where $\rho(\mathbf{z}'; \bm{\theta}) = U(\bz'; \bm{\theta})\ket{0^{\otimes m_0}}\bra{0^{\otimes m_0}}U^\dagger(\bz'; \bm{\theta})$ is the density matrix of the final state. 
Thus, the circuit \(U(\bz'; \bm{\theta})\) constructs a family of functions of the form \eqref{eq35},
which, according to the universal approximation theorem \cite{brylinski2002universal}, can be uniformly approximated on compact sets by single-qubit gates and a CNOT gate, provided the depth \(L\) is sufficient.

A classical post-processing (an affine transformation) projects the probability space \(\mathbb{R}^{2^{m_0}}\) onto the target space \(\mathbb{R}^{m}\) 
\begin{equation}
\label{target}
%		h: \mathbb{R}^{2^{m _0}} \rightarrow \mathbb{R}^{m}, \quad 
f_{\bm{\Theta}}(\bz) = \mathbf{W} \mathbf{p}(F(\mathbf{z}); \bm{\theta}) + \mathbf{b}
\end{equation}
with $\bm{\Theta} = \{\bm{\theta}=(\bm{\theta}_0, \cdots, \bm{\theta}_L), \mathbf{W}\in \mathbb{R}^{m \times 2^{m_0}}, \bm{b}\in \mathbb{R}^{m}\}$.
The cost function \eqref{eq2} admits a variational interpretation as the expectation of the reconstruction error under the empirical data distribution:
\begin{equation}
	\mathcal{L}(\bm{\Theta}) = \mathbb{E}_{(\mathbf{z},\mathbf{x}) \sim \mathcal{D}} \left[ \| \mathbf{x} - (W\mathbf{p}(F(\mathbf{z});\bm{\theta}) + \mathbf{b}) \|_2^2 \right]
\end{equation}

\subsection{Gradient computation}

The gradient with respect to the classical parameters \((W, \mathbf{b})\) is computed via standard automatic differentiation. For the quantum parameters \(\bm{\theta}\), we use the \textit{parameter-shift rule}

\begin{equation}
	\frac{\partial p_j}{\partial \theta_i} = \frac{1}{2} \left( p_j(\theta_i + \tfrac{\pi}{2}) - p_j(\theta_i - \tfrac{\pi}{2}) \right)
\end{equation}
This property follows from the structure of rotation gates and enables an exact analytical computation of the gradients.
Optimization is performed using the Adam (Adaptive Moment Estimation) algorithm.

\section{Results on simulation}

\label{sec5}

\subsection{Synthetic dataset}

To evaluate our HVQC to reconstruct images from input variables, We consider two synthetic dataset (Fig.~\ref{fig4} and \ref{fig5}) composed of $N=900$ samples $(\mathbf{z}^{(i)},\tilde{I}^{(i)})$. Each input variable $\mathbf{z}^{(i)}=(z_1^{(i)},z_2^{(i)})\in\mathbb{R}^2$ parametrizes an image $\tilde{I}^{(i)}$ defined over a discrete spatial grid. The output image  has resolution $16\times16$.% and is represented as a vector in the probability simplex $\Delta^{255}\subset\mathbb{R}^{256}$.

\paragraph{Dataset 1}
The spatial grid is given by two uniform subdivisions \(x_k\) and \(y_l\) of \([-2, 2]\) into 16 points each. Latent variables are sampled from the uniform distribution on \([-1,1]^2\). For each \(\mathbf{z} = (z_1, z_2)\), two vectors are defined on the grid
\begin{align}
v_1(k) = \sin(z_1 g_k)+\cos(z_2 g_k), \
v_2(l) = \cos(z_1 g_l)+\sin(z_2 g_l).
\end{align}

The image is constructed as 
\begin{equation}
\tilde{I}_{\mathbf{z}}(x_k,y_l) =
\frac{|I_{\mathbf{z}}(x_k,y_l)|}
{\sum_{k,l}|I_{\mathbf{z}}(x_k,y_l)|}, \ \  I_{\mathbf{z}}(x_k,y_l) = v_1(k) v_2(l).
\end{equation}
%with $I_{\mathbf{z}}(x_k,y_l) = v_1(k) v_2(l)$.
 
\paragraph{Dataset 2}

Input variables are sampled from $[-2,2]^2$. For a given $\mathbf{z}=(z_1,z_2)$, three spatial patterns are defined:
\begin{align}
h_1(x,y) &= \exp\!\left(-\frac{(x-z_1)^2 + (y-z_2)^2}{2}\right), \\
h_2(x,y) &= \sin(z_1 x + z_2 y), \\
h_3(x,y) &= \exp\!\left(-\frac{(r-\sqrt{z_1^2+z_2^2}/2)^2}{0.5}\right),
\end{align}
where $r=\sqrt{x^2+y^2}$. The resulting image is given by the mixture
\begin{equation}
I_{\mathbf{z}}(x,y) = \alpha h_1(x,y) + \beta h_2(x,y) + (1-\alpha-\beta) h_3(x,y),
\end{equation}
with $\alpha = 0.5 + 0.5 \tanh(z_1)$ and $\beta = 0.5 + 0.5 \tanh(z_2).$
%\begin{equation}
%\alpha = 0.5 + 0.5 \tanh(z_1), \qquad
%\beta = 0.5 + 0.5 \tanh(z_2).
%\end{equation}
As before, the image is normalized.
\paragraph{Dataset 3}
The Friedman1 dataset is a standard synthetic regression benchmark defined by
$y = 10\sin(\pi x_1 x_2) + 20(x_3 - 0.5)^2 + 10x_4 + 5x_5 + \varepsilon$,
with 5 input features uniformly drawn from $[0,1]$ and additive gaussian noise
$\varepsilon \sim \mathcal{N}(0,1)$.
\subsection{Experimental setup}
All experiments were conducted on a simulated quantum environment using PennyLane, with classical components implemented in Tensorflow. The training procedure followed a supervised learning framework as defined, with the Adam optimizer and a learning rate of $10^{-2}$.

For the experiment, we aim to evaluate the impact of the mappings $\bz\mapsto F_1(\bz)=\bz^{\otimes^3}$, $\bz \mapsto F_2(\bz)$ 
where $F_2(\bz)=\mathbf{z}'^{(i)} \in \mathbb{R}^8$ is defined by $F_2(z_1, z_2)=(z_1, z_2, z_1z_2, z_2z_1, z_1^2, z_2^2, z_1^2z_2, z_2^2z_1)$, and $\bz \mapsto F_3(\bz)$ 
where $F_3(\bz)=\mathbf{z}'^{(i)} \in \mathbb{R}^8$ is defined by $F_3(z_1, z_2)=(z_1, z_2, 0, 0, 0, 0, 0, 0)$.

To benchmark the HVQC with 66,032 trainable parameters (240 quantum, 65,792 classical), we selected four classical regression models: 
GPR with an RBF kernel, alpha 0.1, 10 optimizer restarts, and target normalization ; RFR with 100 estimators and default unlimited depth ; and XGB with 100 estimators, maximum depth 6, and learning rate 0.1 ; and two feedforward neural networks with identical architectures (5 hidden layers, ~67,280 parameters) but different activation functions (ReLU and SiLU), trained with the Adam optimizer for up to 1000 epochs.
As noted in Section~\ref{sec3}, all experiments are conducted using statevector simulation, without shot noise.
The impact of shot noise and hardware noise on regression  performance is left for future work.
%All experiments use statevector simulation, providing exact probability estimates without shot noise. This idealized setting avoids the $O(1/\sqrt{S})$ statistical fluctuations inherent to finite sampling, and results should be interpreted accordingly; the impact of shot noise and hardware noise on regression performance is left for future work.
\begin{figure}[htbp]
		\centerline{\includegraphics[scale=0.17]{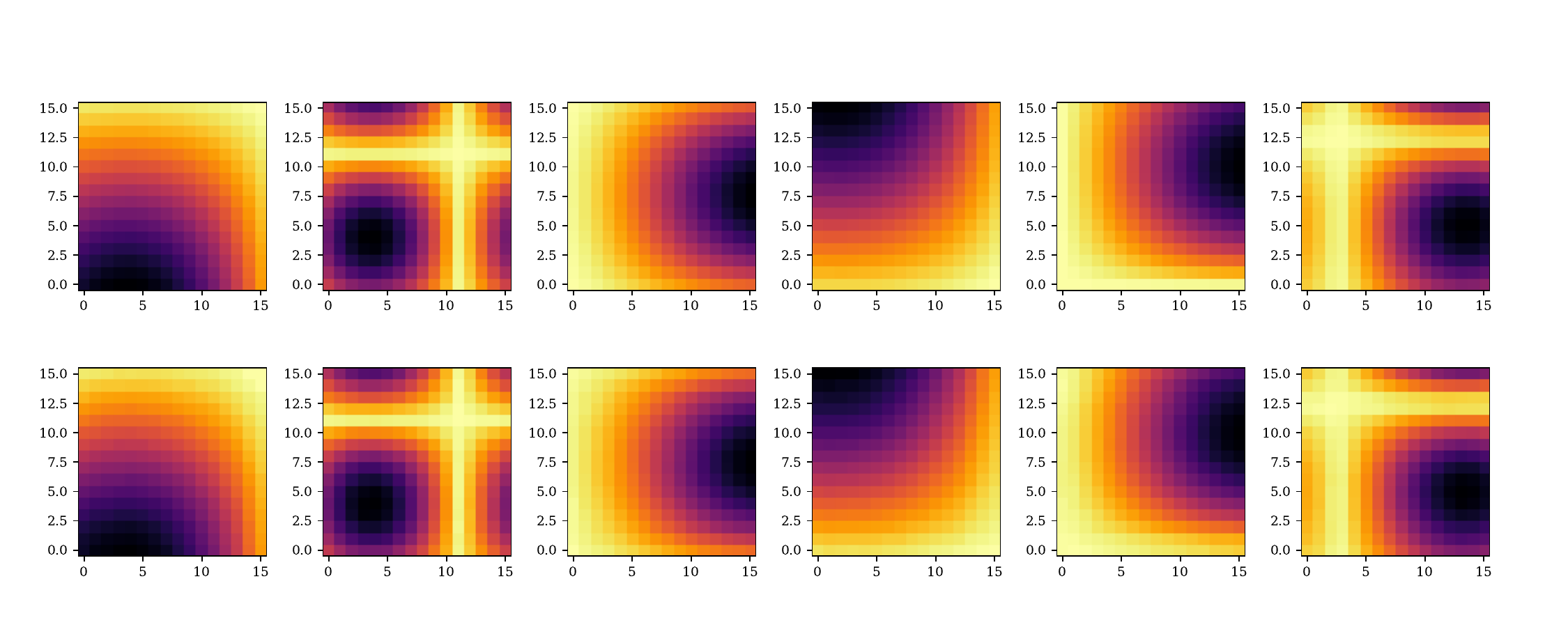}}
		\caption{Dataset 1: Reconstruction of $16\times 16$ maps by the variational quantum circuit whose architecture is described in \eqref{architecture} with number of layers $L=9$.}
		\label{fig4}
\end{figure}
\begin{figure}[htbp]
		\centerline{\includegraphics[scale=0.17]{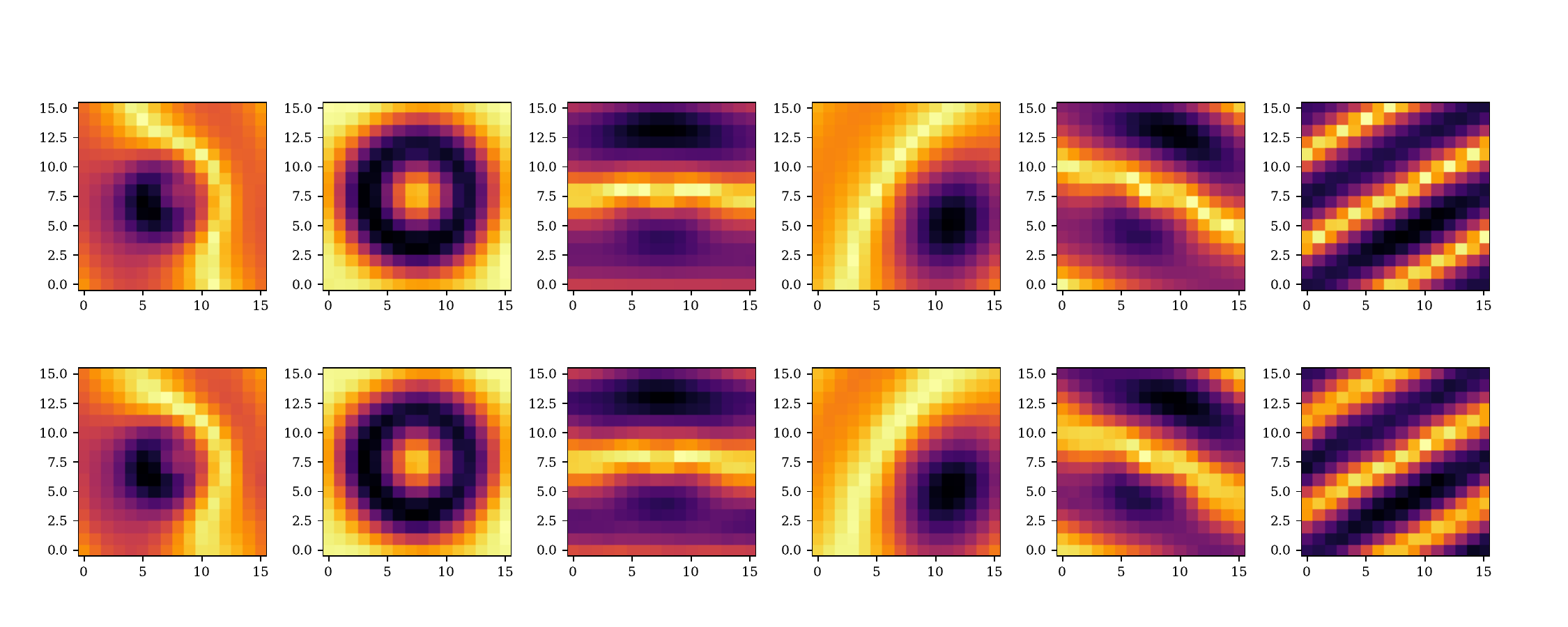}}
		\caption{Dataset 2:Reconstruction of $16\times 16$ maps by the variational quantum circuit whose architecture is described in \eqref{architecture} with number of layers $L=9$.}
		\label{fig5}
\end{figure}

\begin{figure}[htbp]
\centerline{\includegraphics[scale=0.3]{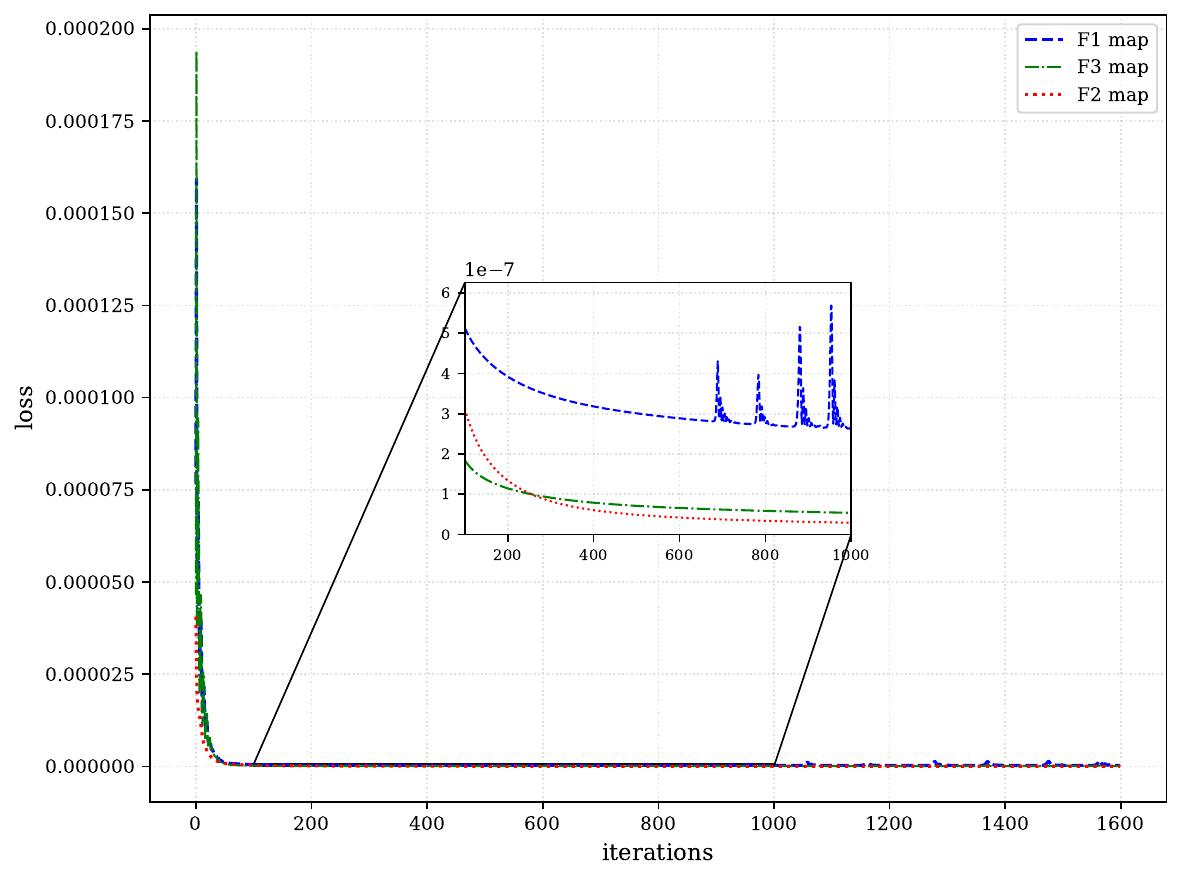}}
		\caption{Learning curve for map reconstruction by the variational quantum circuit \eqref{architecture} with $\bz'=\bz^{\otimes^3}$ in blue and $\bz'=F_1(\bz)$ in red ($L=9$) for dataset 1.}
		\label{fig6}
\end{figure}
\begin{figure}[htbp]
\centerline{\includegraphics[scale=0.3]{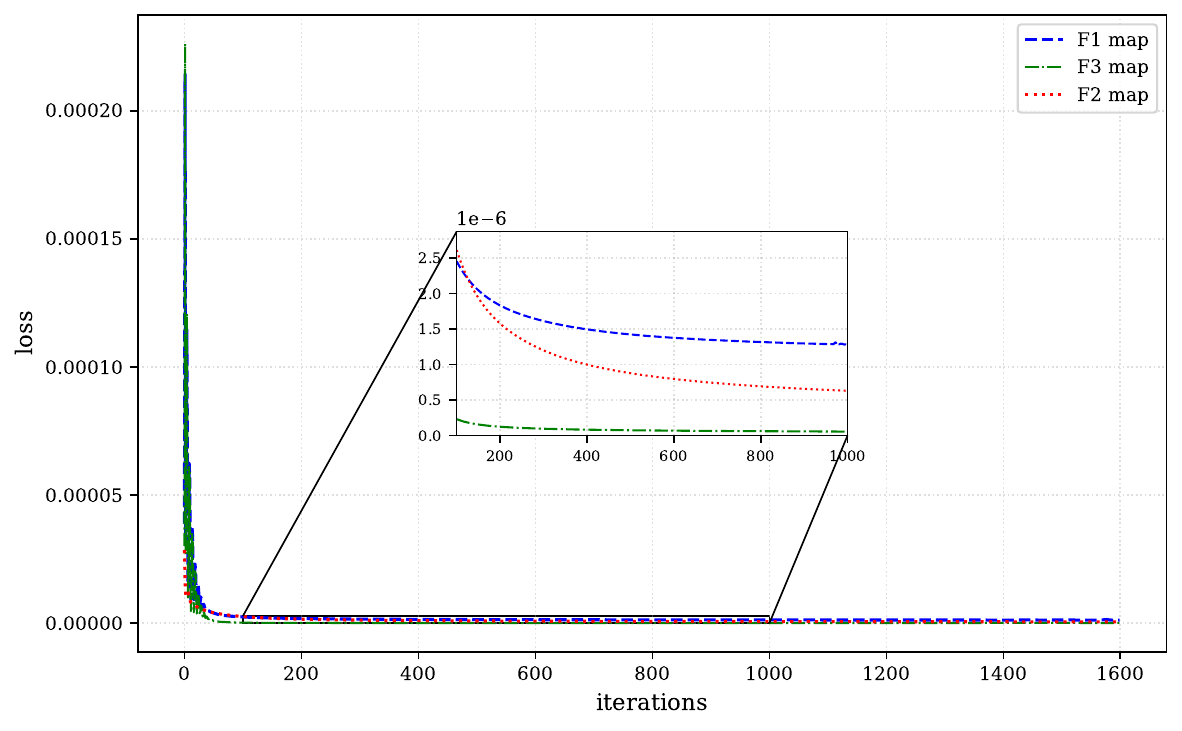}}
		\caption{Learning curve for map reconstruction by the variational quantum circuit \eqref{architecture} with $\bz'=\bz^{\otimes^3}$ in blue and $\bz'=F_1(\bz)$ in red ($L=9$) for dataset 2.}
		\label{fig7}
\end{figure}
\begin{table}[htbp]
		\caption{Comparison of prediction performance for HVQC models with different feature maps and classical (GPR, RFR and XGB)
        across two datasets, using MSE and $R^2$ metrics on train and test splits ( $m_0=2$ qubits and $L=9$ layers).}
		\begin{center}
			\begin{tabular}{|c|c|cc|cc|}
				\hline
				&\textbf{}&\multicolumn{2}{|c|}{\textbf{Dataset 1}}&\multicolumn{2}{|c|}{\textbf{Dataset 2}} \\
                
				\cline{1-6} 
				\textbf{Models}&\textbf{Metrics} & \textbf{\textit{train}}& \textbf{\textit{test}}& \textbf{\textit{train}}& \textbf{\textit{test}} \\
				\cline{1-6} 
				\textbf{HVQC}&\textbf{mse}($\times 10^{-7}$)&2.55 & 3.06 & 5.151&14.398\\
                \cline{2-6} 
				with $F1$ &R$^2$& 0.970& 0.965 & 0.909& 0.762\\
				\cline{1-6} 
                \textbf{HVQC}&\textbf{mse}($\times 10^{-7}$)& $\mathbf{0.218}$& $\mathbf{0.336}$&  1.737&10.262\\
                \cline{2-6} 
				with $F2$ &R$^2$&$\mathbf{0.997}$ & $\mathbf{0.996}$& 0.970& 0.843\\
				\cline{1-6} 
                \textbf{HVQC}&\textbf{mse}($\times 10^{-7}$)& 0.609& 0.745& 0.701 &1.082\\
                \cline{2-6} 
				with $F3$ &R$^2$& 0.991& 0.990&  $\mathbf{0.986}$&$\mathbf{0.978}$\\
                \cline{1-6} 
				\textbf{GPR}&\textbf{mse}($10^{-7}$)& 0.605& 0.788& $\mathbf{0.683}$ &$\mathbf{1.075}$\\
                \cline{2-6} 
				     &R$^2$& 0.993& 0.991&  $\mathbf{0.986}$&$\mathbf{0.978}$\\
                \cline{1-6} 
				\textbf{RFR}&\textbf{mse}($10^{-7}$)& 0.102& 0.782& 0.238 &1.654\\
                \cline{2-6} 
				     &R$^2$& 0.998& 0.995&  0.991&0.968\\
                \cline{1-6} 
				\textbf{XGB}&\textbf{mse}($10^{-7}$)& 0.740& 1.797& 1.944 &4.751\\
                \cline{2-6} 
				     &R$^2$& 0.990& 0.979&  0.960&0.905\\
				
				\cline{1-6} 
			\end{tabular}
			\label{tab1}
		\end{center}
	\end{table}

\subsection{Impact of feature map, ablation study on hybrid architecture, and benchmark against classical models}

The results in Table~\ref{tab1} reveal that the choice of feature map critically influences the predictive performance of the HVQC. Feature map $F_1$ contains redundant components, which degrades its performance. Feature map $F_2$ shows mixed results, performing well on some datasets but poorly on others. In contrast, feature map $F_3$ which retains only the original components and pads the remaining entries with zeros consistently achieves strong predictive performance. Although the affine layer dominates the parameter count (65,792 vs.~240 quantum parameters), the ablation study (Table~\ref{tab:ablation}) shows that removing the quantum circuit reduces the test $R^2$ from $0.978$ to $0.325$ (F3, Dataset~2). Conversely, removing the affine layer (VQC-only) yields negative $R^2$ values, confirming that the final linear projection is equally indispensable to lift the probability simplex constraint.
%This success stems from the hybrid architecture of the HVQC: neither the quantum circuit alone (without post-processing) nor the classical affine layer alone (without the quantum circuit) yields satisfactory results, as demonstrated in our ablation  (Table~\ref{tab:ablation}).

Compared to classical baselines, the HVQC with $F_3$ performs on par with Gaussian Process Regression. Random Forest achieves good performance on Dataset~1 but degrades slightly on Dataset~2, while XGBoost consistently lags behind the other methods. HVQC achieves the best performance ($0.938$) on the Friedman1 benchmark compared to XGB, RFR, and GPR (Table~\ref{tab:friedman1_results}).

\begin{table}[htbp]
\centering
\caption{Ablation study: $R^2$ scores on Dataset~1 (left) and Dataset~2 (right) 
for three configurations: VQC-only (quantum circuit without affine layer), 
Affine-only (classical affine layer applied directly to the encoded input, 
without quantum processing), and Full HVQC. FM denotes feature map.}
\label{tab:ablation}
\begin{tabular}{l l c c c c}
\hline
\multirow{2}{*}{Configuration} & \multirow{2}{*}{FM}& \multicolumn{2}{c}{$R^2$ (dataset 1)} & \multicolumn{2}{c}{$R^2$ (dataset 2)} \\
\cline{3-6}
& & Train & Test& Train & Test \\
\hline
\multirow{3}{*}{VQC-only} & F1  &$-0.747$&$-0.692$& $-1.881$ & $-2.053$ \\
& F2  &$-1.373$&$-1.387$& $-2.244$ & $-2.409$ \\
& F3  &$-0.372$&$-0.348$& $-1.383$ & $-1.512$ \\
\hline
\multirow{3}{*}{Affine-only} & F1 &$0.520$&$0.532$& $0.075$ & $0.151$ \\
& F2 &$0.858$&$0.857$& $0.703$ & $0.651$ \\
& F3 &$0.638$&$0.664$& $0.364$ & $0.325$ \\
\hline
\multirow{3}{*}{Full HVQC} & F1 & $\mathbf{0.970}$ & $\mathbf{0.965}$ & $\mathbf{0.909}$ & $\mathbf{0.762}$ \\
& F2 & $\mathbf{0.997}$ & $\mathbf{0.996}$ & $\mathbf{0.970}$ & $\mathbf{0.843}$ \\
& F3 & $\mathbf{0.991}$ & $\mathbf{0.990}$ & $\mathbf{0.986}$ & $\mathbf{0.978}$ \\
%\multirow{3}{*}{Full HVQC} & F1 &&& $0.909$ & $0.762$ \\
%& F2 &&& $0.970$ & $0.843$ \\
%& F3 &&& $0.986$ & $0.978$ \\
\hline
\end{tabular}
\end{table}

%\noindent The VQC-only configuration fails completely ($R^2 < -1.38$), confirming that the affine layer is necessary for proper geometric projection; the affine-only baseline achieves moderate performance ($R^2 \leq 0.65$), while the full HVQC dramatically outperforms it, with the most striking improvement occurring for feature map F3 ($R^2$ rises from $0.325$ to $0.978$), proving that the quantum circuit learns expressive non-linear features from poor linear inputs.

Two classical neural networks (ReLU and SiLU), with ~67,280 parameters, were trained for comparison. At 100 epochs, they achieve slightly lower performance: for dataset 1, $R^2 \approx 0.994$ (both); for dataset 2, $R^2 \approx 0.975$ (SiLU) and $0.972$ (ReLU). After 1000 epochs, both reach $R^2 \approx 0.999$ (dataset 1) and $0.993$ (dataset 2), matching HVQC performance.
\begin{table}[htbp]
\centering
\caption{XGB, RFR and GPR are evaluated on the Friedman1 dataset using 20 independent splits (200 train / 40\,568 test) with 10-fold cross-validation for hyperparameter selection. For HVQC ( $m_0=5$ qubits and $L=3$ layers, without feature map).}
\label{tab:friedman1_results}
\begin{tabular}{lc}
\hline
\textbf{Method} & \textbf{R$^2$ on test set} \\
\hline
HVQC      & $\mathbf{0.938}$ \\
XGB   & $0.858$ \\
RFR & $0.790$ \\
GPR& $0.934$\\
\hline
\end{tabular}
\end{table}
%Thus, while classical networks can ultimately match or exceed HVQC, they require longer training. HVQC, especially with expressive feature maps like F3 (Fig.~\ref{fig6} and \ref{fig7}) shows better learning efficiency and strong generalization with limited overfitting. Overall, the feature map’s expressivity is more influential than the architecture itself, and hybrid quantum-classical models offer a competitive alternative to classical approaches, particularly when inductive bias or robustness matters.
%To better understand the inner workings of the HVQC, we now analyze the quantum state immediately after measurement, before the final affine transformation.

\subsection{Global analysis of quantum state after measure and implicit clustering behavior}

\begin{figure}[htbp]
	\centerline{\includegraphics[scale=0.24]{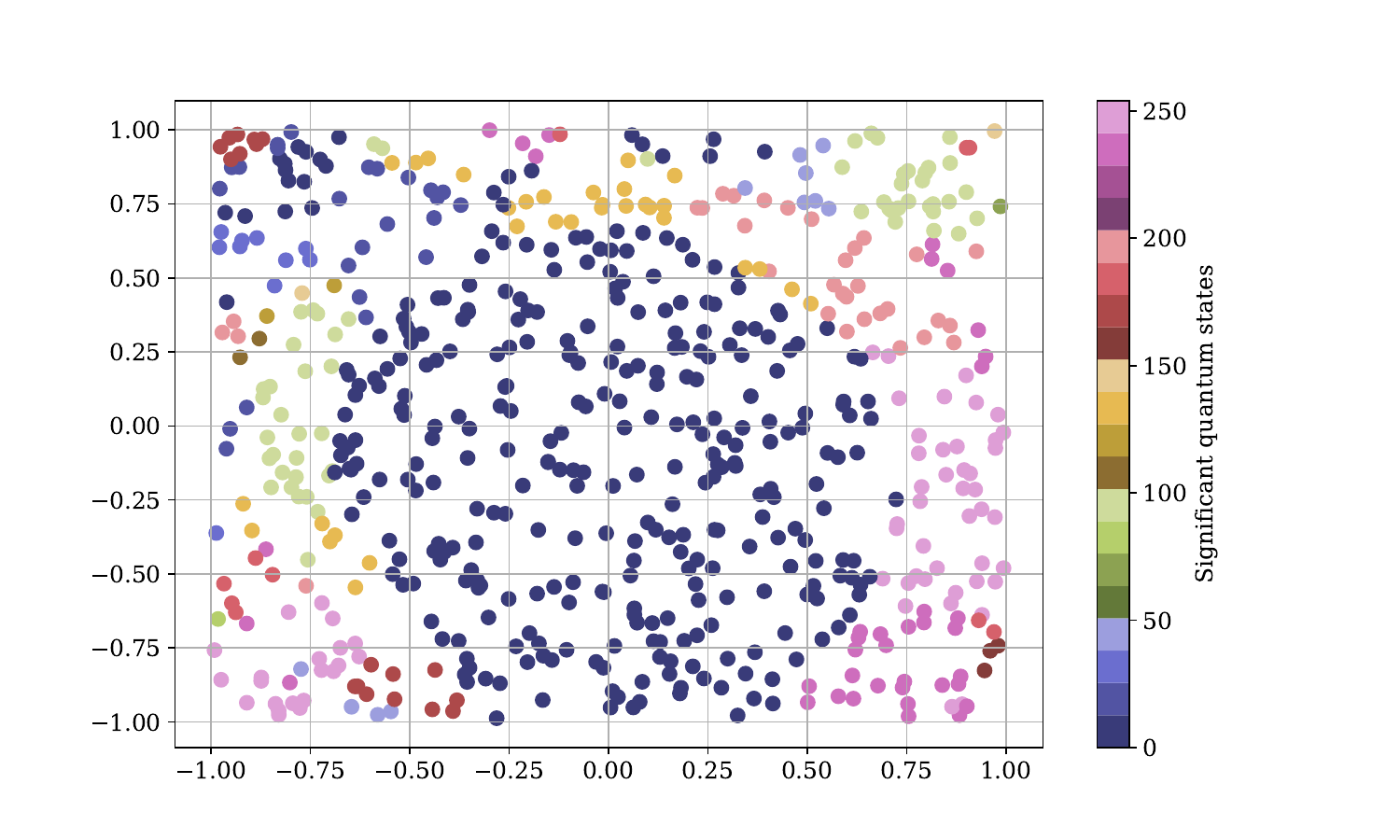}}
	\caption{Emergent activation patterns from HVQC output probabilities before post-processing, on training set of dataset 1.}
	\label{fig8}
\end{figure}

Although the significant states (Fig.~\ref{fig8}) obtained immediately after quantum measurements on all VQC output qubits, representing the hottest pixels, do not form sharply separated clusters according to classical metrics, they reveal a clear overall tendency to group images by category (Fig.~\ref{fig9}). Thus, the VQC implicitly reflects the distribution of threshold-activated elements, facilitating a structured interpretation of the identification of significant pixels across the entire training set of dataset 1.

\begin{figure}[htbp]
	\centerline{\includegraphics[scale=0.3]{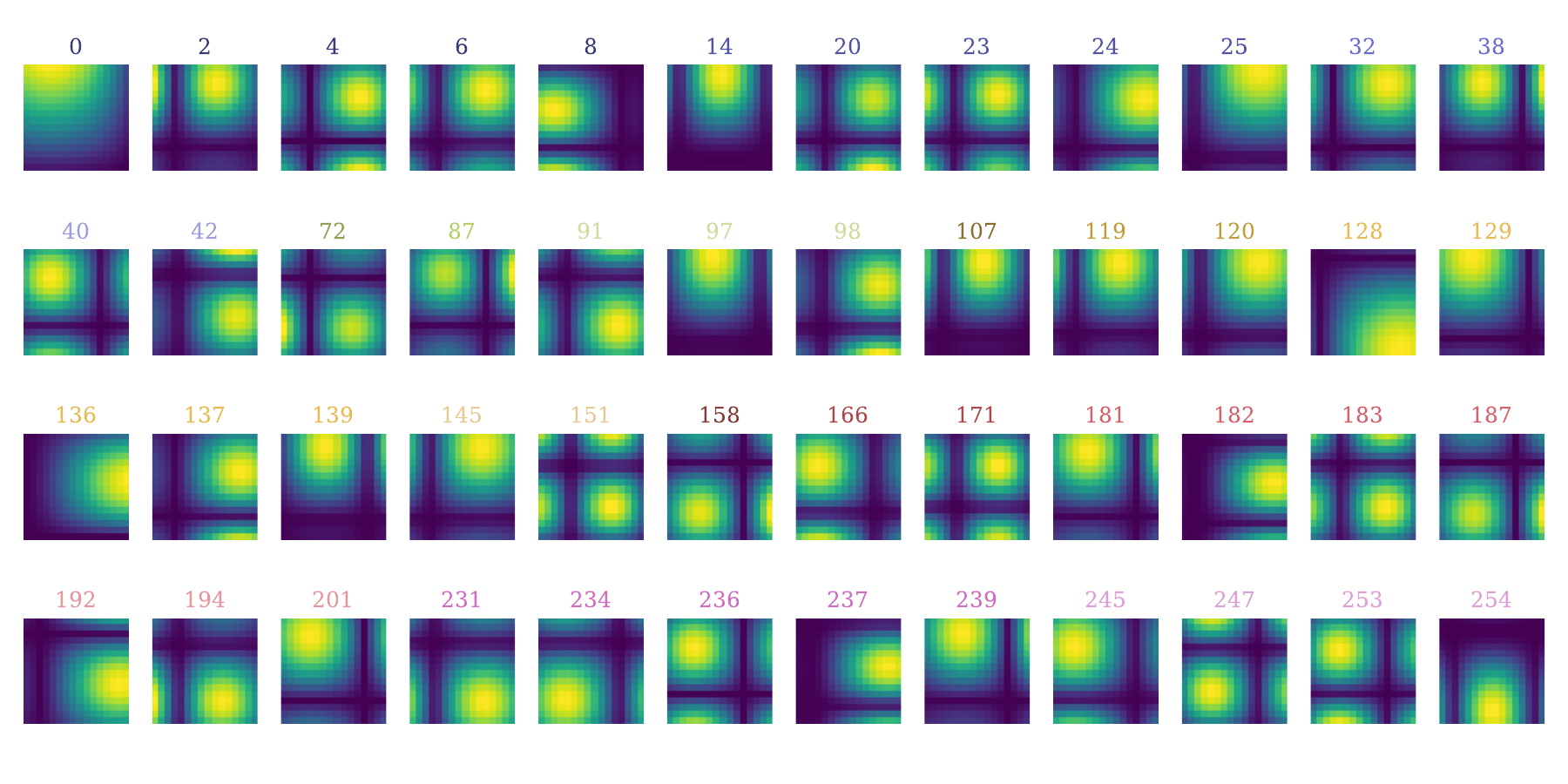}}
	\caption{Sample images from each of 48 clusters of significant quantum state, on training set of dataset 1.}
	\label{fig9}
\end{figure}
\section{Conclusion}
\label{sec6}

We proposed a hybrid variational quantum circuit (HVQC) for 
multivariate regression, combining a parameterized quantum circuit 
with a classical affine post-measurement layer that lifts the 
probability simplex constraint. %Theoretical results establish a formal link between elementary circuit expressivity and the full architecture via data re-uploading, entanglement, and affine post-processing. 
An ablation study confirms that 
both the quantum and classical components are essential, and that 
the architecture follows naturally from elementary circuits via 
data re-uploading, entanglement, and affine post-processing.
%The proposed architecture is derived from the elementary circuit via data re-uploading, entanglement, and affine post-processing.
%An ablation study confirms that both components are essential.
Experimentally, the HVQC matches Gaussian Process Regression on 
synthetic image datasets and achieves $R^2=0.938$ on the larger 
Friedman1 benchmark (40,568 test samples), outperforming XGB and RFR. 
Results highlight the central role of the feature map, whose 
expressivity proves more influential. Future work includes validation on real-world datasets, 
analysis of shot noise and hardware effects, and principled feature 
map selection strategies.
These results suggest that HVQC architectures are a viable 
alternative to classical methods for structured regression tasks, 
particularly when the output space is high-dimensional.

	\bibliographystyle{IEEEtran}
\bibliography{IEEEabrv,ref}
	
\end{document}